\documentclass[letterpaper, 10 pt, conference]{ieeeconf}  

\IEEEoverridecommandlockouts                              
\makeatletter

\let\proof\@undefined
\let\endproof\@undefined

\makeatother

\usepackage{graphicx}
\usepackage{caption}
\usepackage{subcaption}
\usepackage{enumerate}
\usepackage{hyperref}
\usepackage{upgreek}
\usepackage{array}
\usepackage{multirow}
\usepackage{amsmath} 
\usepackage{amssymb}  
\usepackage{amsthm}  
\usepackage{bm}
\usepackage{lipsum}
\usepackage[linesnumbered, ruled]{algorithm2e}
\usepackage{color}
\usepackage{cite}

\newtheorem{proposition}{Proposition}

\newtheorem{lemma}{Lemma}

\newcommand{\dist}{\mathrm{dist}}

\newcommand{\cvar}{\mathrm{CVaR}}

\title{\LARGE \bf
Learning-Based Distributionally Robust Motion Control\\ with Gaussian Processes
\thanks{This work was supported in part by  the Creative-Pioneering Researchers Program through SNU, the Basic Research Lab Program through the National Research Foundation of Korea funded by the MSIT(2018R1A4A1059976), and Samsung Electronics.}
}

\author{Astghik Hakobyan
\and
Insoon Yang 
\thanks{
A. Hakobyan and I. Yang are with the Department of Electrical and Computer Engineering, Automation and Systems Research Institute,  Seoul National University, Seoul 08826, South Korea, 
        {\tt\small \{astghikhakobyan, insoonyang\}@snu.ac.kr}}%
}

\begin{document}

\maketitle
\thispagestyle{empty}
\pagestyle{empty}


\begin{abstract}
Safety is a critical issue in learning-based robotic and autonomous systems as learned information about their environments is often unreliable and inaccurate. 
In this paper, we propose a risk-aware motion control tool that is robust against errors in learned distributional information about obstacles moving with unknown dynamics. 
The salient feature of our model predictive control (MPC) method is its capability of limiting the risk of unsafety even when the true distribution deviates from the distribution estimated by Gaussian process (GP) regression, within an \emph{ambiguity set}.
Unfortunately, the distributionally robust MPC problem with GP is intractable because the worst-case risk constraint involves an infinite-dimensional optimization problem over the ambiguity set. 
To  remove the infinite-dimensionality issue, we develop a  systematic reformulation approach exploiting modern distributionally robust optimization techniques.
The performance and utility of our method are demonstrated through simulations using a nonlinear car-like vehicle model for autonomous driving.
\end{abstract}


\section{Introduction}

The adoption of learning-based decision-making tools for the intelligent operation of
mobile robots and autonomous systems is rapidly growing because of advances in machine learning, sensing, and computing technologies.  
By learning its uncertain and dynamic environment, a robot can use additional information to improve the control performance. 
However, the accuracy of inference is often poor, as it is subject to the quality of the observations, statistical models, and learning methods.
Employing inaccurately learned information in the robot's decision making may cause catastrophic system behaviors, in particular, leading to collision. 
The focus of this work is to develop an optimization-based method
for safe motion control that is robust against errors in learned information about obstacles moving with unknown dynamics.

Learning-based control methods for mobile robots and autonomous systems can be categorized into two classes. 
The first class learns unknown system models, while the second class learns unknown environments.
Control methods that learn unknown system dynamics typically use model predictive control (MPC)~\cite{Aswani2013, di2013stochastic, Ostafew2016, Williams2018, hewing2019cautious} and model-based reinforcement learning (RL)~\cite{Hester2012, venkatraman2016improved, Polydoros2017}.  
These tools employ various learning or inference techniques to update unknown system model parameters that are, in turn, used to improve control actions or policies.
On the other hand, the methods in the second class put more emphasis on ``learning the environment" rather than ``controlling the robot". 
In particular, for learning the behavior (or intention) of obstacles or other vehicles, 
several methods have been proposed that use inverse RL~\cite{Kuderer2015, Herman2015, Wulfmeier2017}, imitation learning~\cite{Kuefler2017, codevilla2018end}, and Gaussian mixture models~\cite{chernova2007confidence, Lenz2017}, among others.

Our method is classified as the second since it learns the movement of obstacles.
However, departing from the previous approaches, we emphasize the importance of ``control"  in correcting potential errors in ``learning".
The key idea is to determine the motion control action that is robust against errors in learned information about the obstacles' motion. 
Specifically, our method uses Gaussian process (GP) regression~\cite{rasmussen2003gaussian} to estimate the probability distribution of the obstacles' locations for future stages based on the current and past observations. 
To actively take into account the possibility that the learned distribution information may be inaccurate, 
we propose a novel MPC method that
optimizes the motion control action subject to
constraints on the \emph{risk of unsafety} evaluated under the worst-case distribution in a so-called \emph{ambiguity set}.
Thus, the resulting control action will satisfy the risk constraints for safety even when the true distribution deviates from the learned one within the ambiguity set. 

Unfortunately, the distributionally robust MPC (DR-MPC) problem is challenging to solve since the worst-case risk constraint involves an infinite-dimensional optimization problem over the ambiguity set of probability distributions. 
To resolve this issue, we propose a reformulation approach using $(i)$ modern distributionally robust optimization techniques based on Kantorovich duality~\cite{esfahani2018data}, $(ii)$ the extremal representation of conditional value-at-risk, and $(iii)$ a geometric expression of the distance to the union of half-spaces. 
The reformulated DR-MPC problem is finite-dimensional and can be efficiently solved by using existing nonlinear programming algorithms.
Through simulations using a nonlinear car-like vehicle model
for collision-avoidance racing, 
we empirically show that, unlike the standard non-robust version, 
our method preserves safety even with moderate errors in the results of GP regression. 

The remainder of the paper is organized as follows. 
In Section~\ref{s3}, we present a GP regression approach to learning the motion of obstacles.
In Section~\ref{s4}, we introduce the learning-based DR-MPC method with a tractable reformulation technique.
The simulation results for collision-avoidance racing are presented
in Section~\ref{s5}.

\section{Learning the Movement of Obstacles Using Gaussian Processes}\label{s3}

\subsection{Obstacle Model}\label{s3ss1}

We consider a rigid body obstacle interfering with the motion of a mobile robot in $\mathbb{R}^{n_y}$. 
The obstacle state $\mathrm{x}_o(t)\in\mathbb{R}^{n_\mathrm{x}}$ is defined as the position and orientation of an arbitrary point on the obstacle. Thus, the obstacle state evolves with
\begin{equation}
\mathrm{x}_o(t+1)=\mathrm{x}_o(t)+T_o \mathrm{v}_o(\mathrm{x}_o(t)), \label{obs_model}
\end{equation}
where $\mathrm{v}_o(\mathrm{x}_o(t))\in\mathbb{R}^{n_\mathrm{x}}$ is the vector of the obstacle's velocity, and $T_o$ is the sample time. For ease of exposition, we describe the case of a single obstacle, but our method is valid in multi-obstacle cases as well.

Having the obstacle's state vector, as well as its geometric parameters, the region occupied by the obstacle at stage $t$ can be modeled (or over-approximated if necessary) as a convex polytope defined by $m$ number of half-spaces:
\begin{equation}
\mathcal{O}(t) :=\{\mathbf{x}\in\mathbb{R}^{n_y} \mid G_{t} \mathbf{x}\leq g_{t}\}.\label{obs_reg}
\end{equation}
Here, $G_{t}\in \mathbb{R}^{m\times n_y}$ and $g_{t}\in\mathbb{R}^m$ are found from the geometry of the obstacle and the current state by $G_t=\mathrm{G}(\mathrm{x}_o(t))$ and $g_t=\mathrm{g}(\mathrm{x}_o(t))$.

For example, for a car-like obstacle in a 2D environment, the state can be chosen as the coordinate and angle of an arbitrary point on the obstacle.
 However, by symmetry, the simplest motion pattern will be obtained for the three candidate states that are shown in Fig.~\ref{obs_def}. The region occupied by the obstacle is over-approximated as a rectangle, the parameters of which can be found using the geometry of the vehicle and any of the three states.
To find the $G_t$ and $g_t$, we need to know the exact expression of $\mathrm{v}_o$. However, in practice it is impossible for a robot to have full knowledge of its environment, in particular, the behavior of the obstacle. 
For predicting the obstacle's motion, 
we use the GP regression approach introduced in the following subsection.

\begin{figure}[t]
  \centering
  \includegraphics[width=0.5\linewidth]{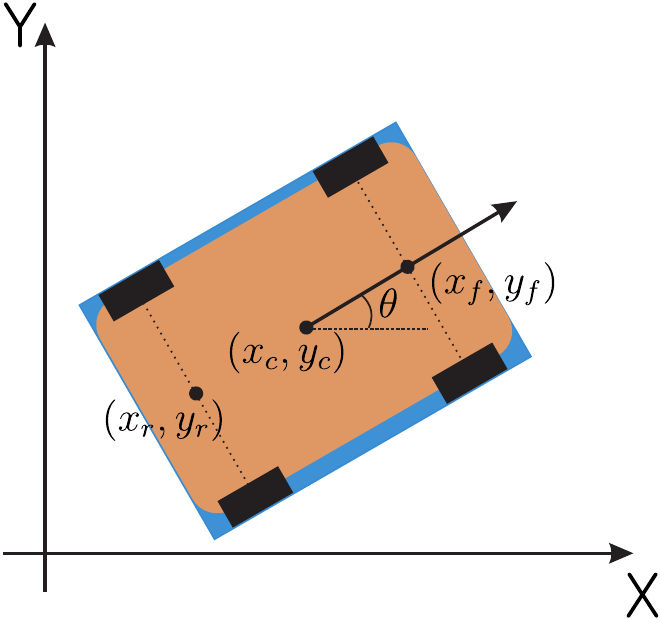}
  \caption{Car-like obstacle in 2D environment. By symmetry, the simplest motion pattern will be obtained for the following three candidates of state: $[x_r,y_r,\theta]^\top$, $[x_f,y_f,\theta]^\top$ and $[x_c,y_c,\theta]^\top$, where $(x_r,y_r)$, $(x_f,y_f)$ and $(x_c,y_c)$ are the coordinates of the center of the rear axle, front axle, and center of mass, respectively, with $\theta$ as the heading angle. The region occupied by the vehicle is over-approximated by the blue rectangle.}
  \label{obs_def}
\end{figure}

\subsection{Gaussian Process Regression}\label{s3ss2}
GP regression is a nonparametric Bayesian approach to regression  and infers a probability distribution over all possible values of a function given some training data~\cite{rasmussen2003gaussian}. 
A GP is a collection of random variables, any finite number of which have a joint Gaussian distribution. 
In this work, GP regression is used for predicting the noisy velocity function $\mathrm{v}_o(\mathrm{x}_o(t))$ from previous observations of the obstacle's behavior.

We choose the training input data as
$\hat{\mathrm{x}}=\{\mathrm{x}_o(t-1),\mathrm{x}_o(t-2),\dots,\mathrm{x}(t-M)\}$, consisting of the obstacle's state for $M$ previous stages. 
The corresponding measured velocities $\hat{\mathrm{v}}$ are selected as the training output data. 
In reality, we do not have access to function values; instead, the following noisy observations are available: for the $i$th observation
\[
\hat{\mathrm{v}}^{(i)}=\mathrm{v}_o(\hat{\mathrm{x}}^{(i)})+\varepsilon, \quad i=1, \ldots, M,
\]
where $\hat{\mathrm{x}}^{(i)} := \mathrm{x}_o (t-i)$, and 
$\varepsilon$ is an i.i.d. zero-mean Gaussian noise with covariance $\Sigma^\varepsilon=\mathrm{diag}([\sigma_{\varepsilon,1}^2\;\sigma_{\varepsilon,2}^2,\dots,\sigma_{\varepsilon,n_\mathrm{x}}^2])$.

 Since the velocities in different dimensions are assumed to be independent, each of them can be learned individually. The  dataset for the $j$th dimension is thus constructed as
\[
\mathcal{D}_j=\big\{\big(\hat{\mathrm{x}}^{(i)},\hat{\mathrm{v}}_j^{(i)}\big),\;i=1,\dots,M\big\}.
\]

For each dimension of output $\mathrm{v}_o(\cdot)$, we specify a GP prior with mean function $m_j(x)$ and kernel function $k_j(x,x')$.  In this paper, we use an RBF kernel that is defined by
\[
k_j(x,x')=\sigma_{f,j}^2\exp\Big[-\frac{1}{2}(x-x')^\top L_j^{-1}(x-x')\Big],
\]
where $L_j$ is a diagonal length scale matrix and $\sigma_{f,j}^2$ is the signal variance.
The prior on the noisy observations is a normal distribution with mean function $m_j(\hat{\mathrm{x}}^{(i)})$ and covariance function $K_j(\hat{\mathrm{x}},\hat{\mathrm{x}})+\sigma_{\varepsilon,j}^2I$, where $K_j(\hat{\mathrm{x}},\hat{\mathrm{x}}) \in \mathbb{R}^{M \times M}$ denotes the  covariance matrix of training input data, i.e., $K_j^{(l,k)}(\hat{\mathrm{x}},\hat{\mathrm{x}})=k_j(\hat{\mathrm{x}}^{(l)},\hat{\mathrm{x}}^{(k)})$.

It follows that the joint distribution of the training output data $\hat{\mathrm{v}}_j$ and the output $\mathbf{v}_j$ at an arbitrary test point $\mathbf{x}$ is given by
\[
\begin{bmatrix}
\hat{\mathrm{v}}_j \\ \mathbf{v}_j
\end{bmatrix}\sim \mathcal{N}\bigg(\begin{bmatrix}m_j(\hat{\mathrm{x}})\\m_j(\mathbf{x})\end{bmatrix},\begin{bmatrix}K_j(\hat{\mathrm{x}},\hat{\mathrm{x}})+\sigma_{\varepsilon,j}^2I & K_j(\hat{\mathrm{x}},\mathbf{x})\\ K_j(\mathbf{x},\hat{\mathrm{x}}) & k_j(\mathbf{x},\mathbf{x}) \end{bmatrix}\bigg),
\]
where $K_j^{(l)}(\hat{\mathrm{x}},\mathbf{x}) =k_j(\hat{\mathrm{x}}^{(l)},\mathbf{x})$, and $K_j(\mathbf{x},\hat{\mathrm{x}}) = K_j(\hat{\mathrm{x}},\mathbf{x})^\top$.
As a result, the posterior distribution of the output in the $j$th dimension at an arbitrary test point $\mathbf{x}$ conditioned on the observed data is Gaussian, with the following mean and covariance:
\begin{align}
&\bm{\mu}_\mathrm{v}^j(\mathbf{x}) : =m_j(\mathbf{x}) \nonumber\\
& \qquad +K_j(\mathbf{x},\hat{\mathrm{x}})(K_j(\hat{\mathrm{x}},\hat{\mathrm{x}})+\sigma_{\varepsilon,j}^2I)^{-1}(\hat{\mathrm{v}}_j-m_j(\hat{\mathrm{x}})),\label{mean}\\
&\bm{\Sigma}_\mathrm{v}^j(\mathbf{x})=k_j(\mathbf{x},\mathbf{x}) \nonumber\\
& \qquad -K_j(\mathbf{x},\hat{\mathrm{x}})(K_j(\hat{\mathrm{x}},\hat{\mathrm{x}})+\sigma_{\varepsilon,j}^2I)^{-1}K^j(\hat{\mathrm{x}},\mathbf{x}).\label{cov}
\end{align}

\begin{figure*}[t!]
\centering
\begin{subfigure}{0.29\linewidth}
  \centering
  \includegraphics[width=\linewidth]{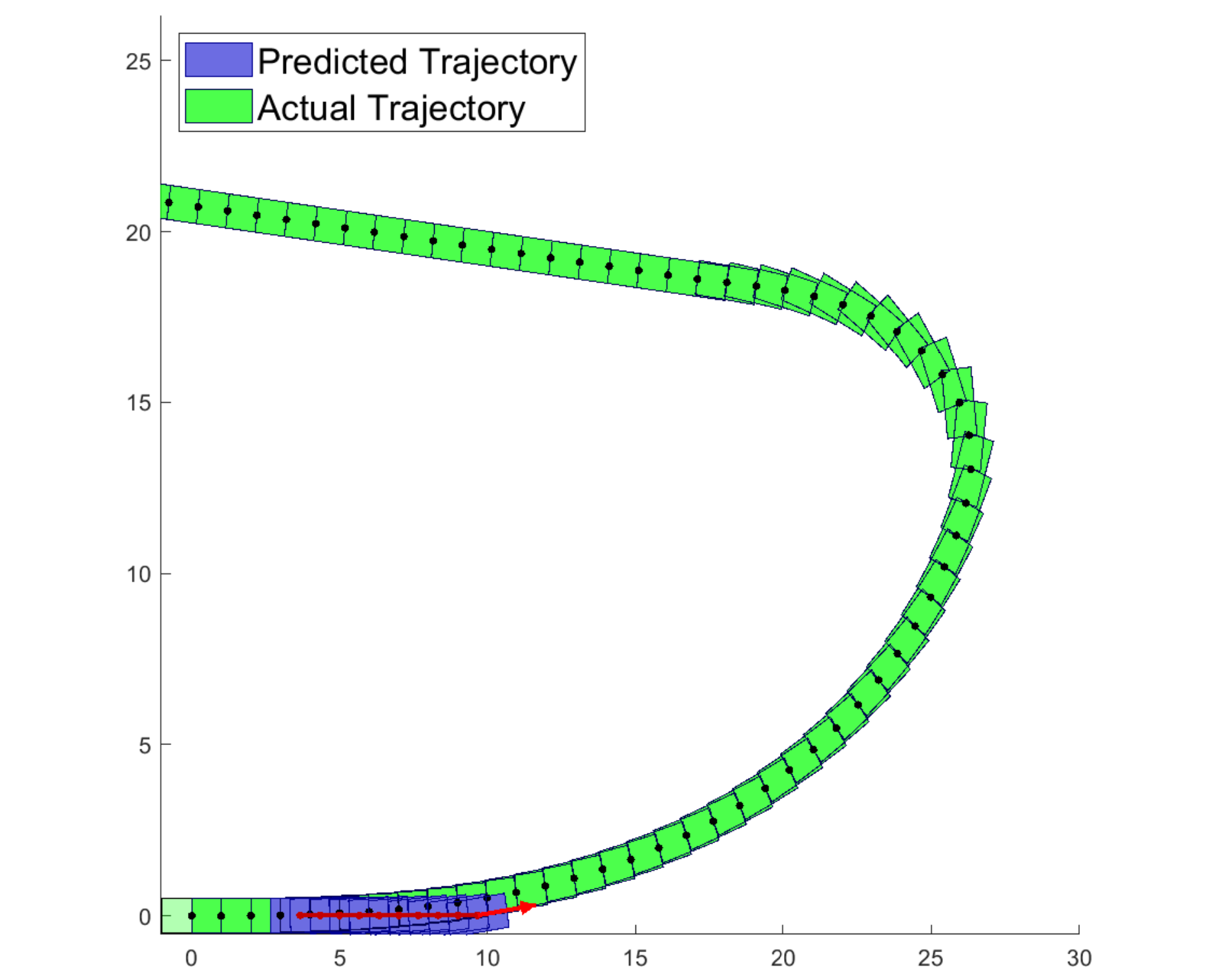}
  \subcaption{$t=5$}
  \label{GP_fig1}
\end{subfigure}%
\begin{subfigure}{0.29\linewidth}
  \centering
  \includegraphics[width=\linewidth]{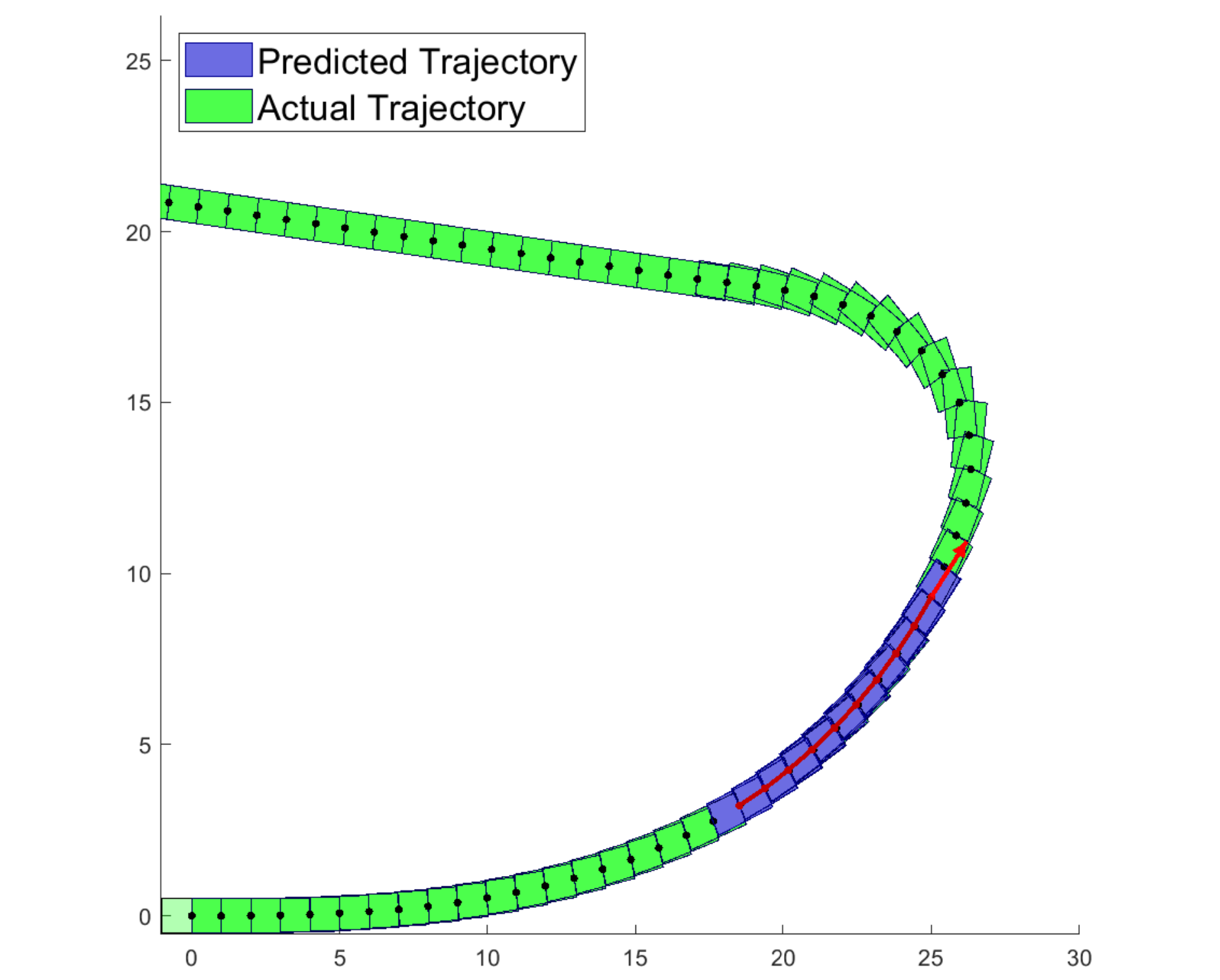}
    \subcaption{$t=20$}
    \label{GP_fig2}
\end{subfigure}%
\begin{subfigure}{0.29\linewidth}
  \centering
  \includegraphics[width=\linewidth]{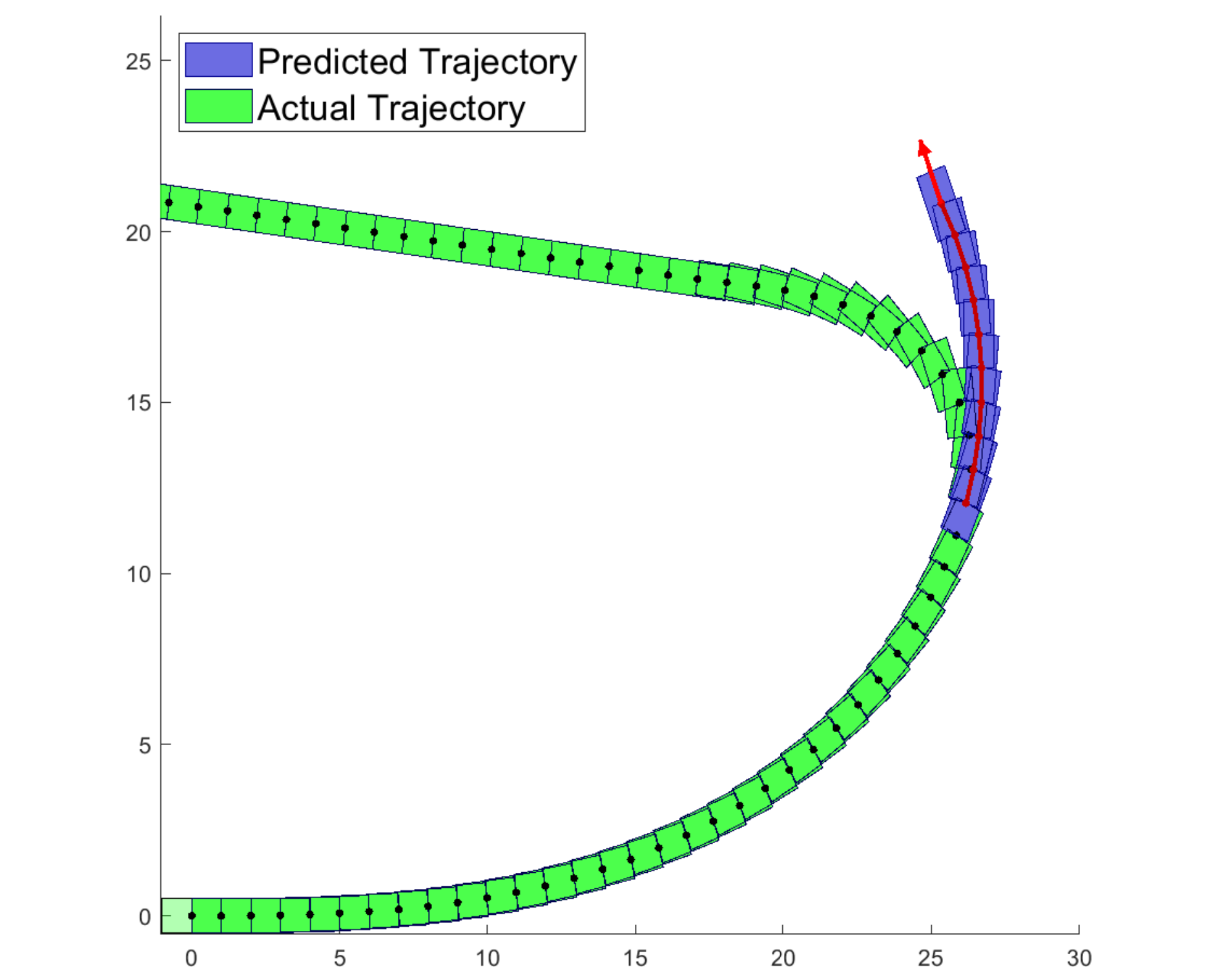}
    \subcaption{$t=32$}
    \label{GP_fig3}
\end{subfigure}
\caption{Predicted mean trajectories of a car-like obstacle for the next $10$ stages.}
\label{GP_fig}
\vspace{-0.15in}
\end{figure*}

The resulting GP approximation of  $\mathrm{v}_o$ is then given by
\[
\mathbf{v}(\mathbf{x})\sim\mathcal{GP}(\bm{\mu}_\mathrm{v}(\mathbf{x}), \bm{\Sigma}_\mathrm{v}(\mathbf{x})),
\]
where $\bm{\mu}_\mathrm{v}(\mathbf{x})=[ \bm{\mu}_\mathrm{v}^1(\mathbf{x}),\dots, \bm{\mu}_\mathrm{v}^{n_\mathrm{x}}(\mathbf{x})]^\top$, and $\bm{\Sigma}_\mathrm{v}(\mathbf{x})=\mathrm{diag}([\bm{\Sigma}_\mathrm{v}^1(\mathbf{x}),\dots, \bm{\Sigma}_\mathrm{v}^{n_{\mathrm{x}}}(\mathbf{x})])$.

\subsection{Prediction of Obstacle's Motion}\label{s3ss3}

Assuming that  $\mathrm{x}_o(0)\sim\mathcal{N}(\mu_\mathrm{x}(0),\mathbf{0})$, it is straightforward to check that $\mathrm{x}_o(t)$ is normally distributed at each stage $t$ with mean $\mu_\mathrm{x}(t)$ and covariance $\Sigma_\mathrm{x}(t)$ to be specified. Having the posterior of the velocity vector, the state of the obstacle at the next stage can be predicted by considering the following joint distribution of the state and velocity vectors~\cite{hewing2019cautious}:
\[
\begin{bmatrix}
\mathrm{x}_o(t)\\ \mathbf{v}(t)
\end{bmatrix}\sim\mathcal{N}\Big(
\begin{bmatrix}
\mu_\mathrm{x}(t)\\
\mu_\mathrm{v}(t) 
\end{bmatrix},\begin{bmatrix}\Sigma_\mathrm{x}(t) & \Sigma_\mathrm{xv}(t)\\
\Sigma_\mathrm{vx}(t) & \Sigma_\mathrm{v}(t)\end{bmatrix}\Big).
\]

Following procedures in~\cite{girard2002gaussian} and \cite{girard2003gaussian} and applying the first-order Taylor approximation to \eqref{mean} and \eqref{cov} with Gaussian input $\mathrm{x}_o(t)\sim\mathcal{N}(\mu_{\mathrm{x}}(t),\Sigma_{\mathrm{x}}(t))$ yields the following approximate mean and covariance functions:
\begin{equation}
\begin{split}
&\tilde{\mu}_{\mathrm{v}}(t)= \bm{\mu}_{\mathrm{v}}(\mu_{\mathrm{x}}(t))\\
&\tilde{\Sigma}_{\mathrm{v}}(t)= \bm{\Sigma}_{\mathrm{v}}(\mu_{\mathrm{x}}(t))+\nabla \bm{\mu}_{\mathrm{v}}(\mu_{\mathrm{x}}(t))\Sigma_{\mathrm{x}}(t) \nabla \bm{\mu}_{\mathrm{v}}(\mu_{\mathrm{x}}(t))^\top\\
&\tilde{\Sigma}_{\mathrm{xv}}(t)=\Sigma_{\mathrm{x}}(t)\nabla \bm{\mu}_{\mathrm{v}}(\mu_{\mathrm{x}}(t))^\top.
\end{split}\label{Taylor_app}
\end{equation}

Now, it follows from \eqref{obs_model} that the obstacle's state at the next stage is also normally distributed with the following  mean and covariance:
\begin{equation}
\begin{split}
&{\mu}_{\mathrm{x}}(t+1)= {\mu}_{\mathrm{x}}(t)+T_o {\mu}_{\mathrm{v}}(t)\\
&{\Sigma}_{\mathrm{x}}(t+1)=  {\Sigma}_{\mathrm{x}}(t)+T_o^2 {\Sigma}_{\mathrm{v}}+T_o({\Sigma}_{\mathrm{xv}}+{\Sigma}_{\mathrm{vx}}).
\end{split}\label{update_eq}
\end{equation}
Using \eqref{Taylor_app} and \eqref{update_eq},
the approximate mean and variance of $\mathrm{x}_o(t)$ can be updated.

Having the inferred or predicted obstacle state $\mathrm{x}_o(t)$,
 it is straightforward to obtain $g_t$ and $G_t$ in \eqref{obs_reg}
 as $g_t=\mathrm{g}(\mathrm{x}_o(t))$ and $G_t=\mathrm{G}(\mathrm{x}_o(t))$.
An example of predicting the motion of an obstacle is shown in Fig.~\ref{GP_fig}, where a car-like vehicle is chosen as the obstacle with unknown dynamics.  GP regression is used to predict the trajectory of the vehicle for the next 10 stages. 
As shown in Fig.~\ref{GP_fig1}, the predicted mean in an early stage ($t = 5$) deviates from the actual trajectory, as there were no observations available. 
As more data are collected, the robot better learns the motion pattern of the car-like obstacle. 
As a result, in Figures~\ref{GP_fig2} the difference between the predicted mean and the actual one is small.

However, in practice, the motion predicted by GP regression can be quite different from the actual movement of an obstacle, for example, when it abruptly changes the heading angle, as in the case of Fig.~\ref{GP_fig3}. 
To guarantee safety even when learning fails, we propose a distributionally robust motion control tool  in the following section.

\section{Learning-Based Distributionally Robust Motion Control}\label{s4}

Consider a mobile robot navigating in $\mathbb{R}^{n_y}$ according to the following discrete-time dynamics:
\begin{align}
\xi(t+1)=f(\xi(t),u(t))\\
y(t)=h(\xi(t),u(t)),
\end{align}
where $\xi(t)\in \mathbb{R}^{n_{\xi}}$ and $u(t)\in\mathbb{R}^{n_u}$ are the robot's state and control inputs, respectively, and $y(t)\in\mathbb{R}^{n_y}$ is the robot's current position in the $n_y$-dimensional configuration space.
At stage $t$, the robot is subject to the following state and control constraints:
\begin{equation}
\xi(t)\in\Xi (t),\quad u(t)\in\mathcal{U}(t),\label{feas_const}
\end{equation}
where $\Xi(t)\subseteq\mathbb{R}^{n_\xi}$ and $\mathcal{U}(t)\subseteq\mathbb{R}^{n_u}$.

The robot's environment changes over time as the obstacle moves according to its unknown dynamics. As introduced in our previous work \cite{hakobyan2019risk}, the \emph{safe region} regarding the obstacle is defined by the complement of the region occupied by it, i.e.
\[
\mathcal{Y}(t):=\mathbb{R}^{n_y}\setminus \mathcal{O}^o(t)\quad \forall t\geq 0,
\]
where $\mathcal{O}^o$ denotes the interior of $\mathcal{O}(t)$.
Our goal is to control the robot while keeping it in the \emph{safe region}, even when the GP-based prediction results are inaccurate.

\subsection{Risk Constraint for Safety}\label{s4ss1}

To systematically measure the risk of collision, we use the notion of \emph{safety risk} introduced in our previous work~\cite{Samuelson2018}.
We first define the \emph{loss of safety} as the deviation of the robot's position from the safe region $\mathcal{Y}(t)$:
\begin{equation}
\dist(y(t),\mathcal{Y}(t)):=\min_{a\in\mathcal{Y}(t)}\|y(t)-a\|_2. \label{dist1}
\end{equation}
For safety, it is ideal to force the robot to stay inside the safe region. However, due to the uncertain movement of the obstacle, such a deterministic approach is often too conservative or infeasible. 
Instead, we employ the conditional value-at-risk (CVaR)~\cite{Rockafellar2002a} to define the safety risk at stage $t$ as 
\[
\cvar_\alpha[\dist(y(t),\mathcal{Y}(t))],
\]
where $\cvar_\alpha (X) := \min_{z \in \mathbb{R}} \mathbb{E} \big [
z + (X-z)^+/(1-\alpha)
\big ]$.\footnote{We let $(\bm{x})^+ := \max \{\bm{x}, 0\}$ throughout this paper.}
The safety risk quantifies the average loss of safety beyond the confidence level $\alpha$. 
Note that CVaR is a \emph{coherent} risk measure in the sense of Artzner {\it et al.}~\cite{Artzner1999} and thus satisfies \emph{axioms} that risk metrics in robotics applications should respect for rationally assessing risk~\cite{Majumdar2017isrr}.
More importantly, CVaR is able to distinguish the worst-case tail events, which is crucial for quantifying rare but unsafe events.

The desired level of safety can be reached by limiting the safety risk by a pre-specified risk tolerance parameter $\delta$:
\begin{equation}
\cvar_\alpha[\dist(y(t),\mathcal{Y}(t))]\leq\delta. \label{orig_cvar}
\end{equation}
This risk constraint is adopted in our MPC for safe motion control in the following subsection. 

\subsection{Wasserstein Distributionally Robust GP-MPC}

The safe region in~\eqref{orig_cvar} depends on $g_{t}$ and $G_{t}$, which define the region occupied by the obstacle at stage $t$. Unfortunately, the distribution of these two parameters is unknown and challenging to directly identify in practice. 
However, having sample data $\{\tilde{\mathrm{x}}^{(1)}_o(t),\tilde{\mathrm{x}}^{(2)}_o(t),\dots,\tilde{\mathrm{x}}^{(N)}_o(t)\}$ generated according to the learned distribution of $\mathrm{x}_o(t)$, it is possible to obtain a sample of $g_{t}$ and $G_{t}$ using
\begin{equation}
\tilde{g}^{(i)}_{t}:=\mathrm{g}(\tilde{\mathrm{x}}^{(i)}_o(t)),\quad \tilde{G}^{(i)}_{t}:=\mathrm{G}(\tilde{\mathrm{x}}^{(i)}_o(t)).\label{gG_eq}
\end{equation}

We can then use the sample data to approximate the safety risk in \eqref{orig_cvar}.
However, making such an approximation using limited data may lead to the violation of the original risk constraint~\eqref{orig_cvar}. 
Instead of directly using the learning result of GP regression, we proposed a motion control method that is robust against errors in the estimated distribution.

For a concrete MPC formulation, we first rewrite the loss of safety~\eqref{dist1} in an equivalent form  using the definition of the safe region, which is a union of half-spaces.

\begin{lemma}\label{lem:dist}
Suppose that the region occupied by the obstacle is given by \eqref{obs_reg}. Then, the loss of safety~\eqref{dist1} can be expressed as
\begin{equation}
\dist(y(t),\mathcal{Y}(t))=\min_{j=1,\dots,m}\bigg\{\frac{\big(g_{t,j}-G_{t,j} y(t)\big)^+}{\|G_{t,j}\|_2}\bigg\}, \label{dist}
\end{equation} 
where $g_{t,j}$ is the $j$th element of $g_t$, and  $G_{t,j}$ is the $j$th row of $G_{t}$. 
\end{lemma}
\begin{proof}
The proof is similar to the proof of \cite[Lemma 1]{hakobyan2020wasserstein}, which we briefly summarize here.
Since the safe region is a union of half-spaces, the distance can be written as the shortest distance to the all half-spaces that define the safe region:
\begin{equation}
\dist(y(t),\mathcal{Y}(t))=\min_{j=1,\dots,m}\dist(y(t),\mathcal{Y}_j(t))\label{dist_half}
\end{equation}
where $\mathcal{Y}_j(t)=\{\mathbf{x}\mid G_{t,j}\mathbf{x}\geq g_{t,j}\}$.
The distance to each half-space can then be expressed in its dual form as
\begin{equation}
\dist(y(t),\mathcal{Y}_j(t))=\bigg(\frac{g_{t,j}-G_{t,j} y(t)}{\|G_{t,j}\|_2}\bigg)^+\label{dist_half1}
\end{equation} 
using an argument similar to the proof of \cite[Lemma~1]{hakobyan2020wasserstein}.
Strong duality follows from the fact that the primal problem is feasible and the inequality constraints are linear. 
By substituting \eqref{dist_half1} into \eqref{dist_half}, the result follows.
\end{proof}

We now let 
\begin{equation}\label{cd}
c_{t,j} :=\frac{G_{t,j}}{\|G_{t,j}\|_2}, \quad d_{t,j} := \frac{g_{t,j}}{\|G_{t,j}\|_2}. 
\end{equation}
Using the sample data~\eqref{gG_eq} of $\tilde{g}_t^{(i)}$ and $\tilde{G}_t^{(i)}$, we can then generate a sample $\{(\tilde{c}_{t,j}^{(i)}, \tilde{d}_{t,j}^{(i)})\}_{i=1}^N$ of $(c_{t,j}, d_{t,j})$ according to the definition above.
Let $\mathrm{Q}_t$ be the joint empirical distribution of $(c_{t}, d_{t}) \in \mathbb{W} \subseteq \mathbb{R}^{m (n_y+1)}$ constructed using the sample data, i.e.,
$\mathrm{Q}_t := \sum_{i=1}^N \bm{\delta}_{(\tilde{c}_{t}^{(i)}, \tilde{d}_{t}^{(i)})}$,
where $\bm{\delta}_{\bm{x}}$ denotes the Dirac delta measure concentrated at~$\bm{x}$.
However, the accuracy of the empirical distribution is subject to errors in the learning results. 
To satisfy the risk constraint~\eqref{orig_cvar} even under distribution errors, 
we instead impose the following distributionally robust risk constraint:
\begin{equation}
\sup_{\mathrm{P}_t\in\mathbb{D}_{t}}\cvar_{\alpha}^{\mathrm{P}_t}[\dist(y(t),\mathcal{Y}(t))]\leq \delta. \label{DR_const}
\end{equation}
Here, the left-hand side of the inequality represents the worst-case  CVaR when the joint distribution $\mathrm{P}_{t}$ of $(c_t, d_t)$ lies in a given \emph{ambiguity set} $\mathbb{D}_{t}$. 
Thus, any motion control action that satisfies \eqref{DR_const} can meet the original risk constraint under any distribution error characterized by $\mathbb{D}_t$.
In this work,
we use the following \emph{Wasserstein ambiguity set}: 
\begin{equation}
\mathbb{D}_t:=\{\mathrm{P}\in \mathcal{P}(\mathbb{W})\mid W(\mathrm{P},\mathrm{Q}_t)\leq \theta\},
\end{equation}
where $\mathcal{P}(\mathbb{W})$ denotes the set of Borel probability measures on the support $\mathbb{W}$. 
Here,   $W (\mathrm{P}, \mathrm{Q})$ is the
Wasserstein distance (of order 1) between $\mathrm{P}$ and $\mathrm{Q}$, defined by
\begin{equation} \nonumber
\begin{split}
W (\mathrm{P}, \mathrm{Q}) := \min_{\kappa \in \mathcal{P}(\mathbb{W}^2)} \bigg \{
& \int_{\mathbb{W}^2} \| w - w' \|_2 \; \mathrm{d} \kappa (w, w') \\
 &\mid \Pi^1 \kappa = \mathrm{P}, \Pi^2 \kappa = \mathrm{Q}
\bigg \},
\end{split}
\end{equation}
where $\Pi^i \kappa$ denotes the $i$th marginal of $\kappa$ for $i=1, 2$, where $n:= m (n_y+1)$ is the dimension of $(c_t, d_t)$.
The Wasserstein distance between two probability distributions represents the minimum cost of redistributing mass from one  to another using a non-uniform perturbation.
Using the Wasserstein metric in
distributionally robust optimization and control has recently drawn a great deal of interest because it provides a tractable solution with superior statistical properties such as a probabilistic out-of-sample performance guarantee~\cite{esfahani2018data,Gao2016, Zhao2018, Blanchet2018, Yang2017lcss, Yang2018}.

Using the distributionally robust risk constraint~\eqref{DR_const},
we formulate the following MPC problem:
\begin{subequations}\label{DRMPC}
\begin{align}
\inf_{\mathbf{u, \bm{\xi}, y}} \; & J(\xi(t),\mathbf{u}):=\sum_{k=0}^{K-1} r(\xi_k,u_k)+q(x_K)\label{DRMPCcost}\\
\mathrm{s.t.} \; & \xi_{k+1}=f(\xi_k,u_k) \label{DRMPCcons1}\\
&y_k=h(\xi_k,u_k) \label{DRMPCcons2}\\
&\xi_0=\xi(t)\label{DRMPCcons3}\\
& \sup_{\mathrm{P}_k\in\mathbb{D}_{k}}\; \cvar_{\alpha}^{\mathrm{P}_k}[\dist(y_k,\mathcal{Y}_k)]\leq\delta\label{DRMPCcons4}\\
&\xi_k\in \Xi,\; u_k\in \mathcal{U}, \label{DRMPCcons5}
\end{align}
\end{subequations}
where $\bold{u} := (u_0, \ldots, u_{K-1})$, $\bm{\xi} := (\xi_0, \ldots, \xi_K)$, $\bold{y}:= (y_0, \ldots, y_{K})$, constraint \eqref{DRMPCcons1} and $u_k \in \mathcal{U}$ in \eqref{DRMPCcons5} should hold for $k=0,\dots,K-1$, \eqref{DRMPCcons2} should hold for $k=0,\dots,K$, and all the remaining constraints should be satisfied for $k=1,\dots,K$.
Note that the problem can be extended to consider $L$ obstacles by repeating the constraints \eqref{DRMPCcons4} $L$ times.

The distributionally robust MPC (DR-MPC) problem with GP  is defined in a receding horizon manner for each stage. 
The cost function can be chosen in a way that would guide the robot so it follows a reference trajectory $y^{ref}$ generated by, for example, RRT*~\cite{Karaman2011}:
\begin{equation}
J(\xi(t),\bold{u}):=\|y_{K}-y^{ref}_{K}\|_P+\sum_{k=0}^{K-1}\|y_{k}-y^{ref}_{k}\|_Q+\|u_{k}\|_R,\label{cost}
\end{equation}
where $Q\succeq 0$, $R\succ 0$ are the state and control weighting matrices, respectively; and $P\succeq 0$  is chosen in a way   to ensure stability.
The constraints \eqref{DRMPCcons1} and \eqref{DRMPCcons2} are used for computing the robot state and output over the MPC horizon, specifying the initial state $\xi_0$ as the current state $\xi(t)$ in the constraint \eqref{DRMPCcons3}. 
Most importantly, \eqref{DRMPCcons4} corresponds to the distributionally robust risk constraint, thereby limiting the safety risk by a pre-specified tolerance even when the actual distribution deviates from the  distribution estimated by GP regression within $\mathbb{D}_k$. 
Here, the Wasserstein ambiguity set $\mathbb{D}_k$ is constructed from the joint empirical distribution $\mathrm{Q}_k$ of $(c_k,d_k)$ at each time step $k$. The joint distribution is obtained from GP regression, by learning the obstacle's velocity vector and evolving the obstacle's state according to \eqref{update_eq} from $\mathrm{x}_o(t)$ to $\mathrm{x}_o(t+K)$.
Finally, \eqref{DRMPCcons5} are the state and control constraints given in \eqref{feas_const}.

\subsection{Tractable Reformulation}\label{s4ss2}

Unfortunately, solving the DR-MPC problem~\eqref{DRMPC} is a challenging task because the risk constraint~\eqref{DRMPCcons4} involves an infinite-dimensional optimization problem over the ambiguity set of probability distributions. 
To resolve this issue, we reformulate the DR-MPC problem in a computationally tractable form. 

To begin with, we make use of Lemma~\ref{lem:dist}
 to rewrite the safety risk as
\begin{align}
&\cvar_\alpha[\dist(y,\mathcal{Y})]=\min_{z\in\mathbb{R}}\mathbb{E}\bigg[z+\frac{(\dist(y,\mathcal{Y})-z)^+}{1-\alpha}\bigg]\nonumber\\
&=\min_{z\in\mathbb{R}} \bigg \{z+\mathbb{E}\bigg[\frac{\max\{\min_{j} (c_{j}y + d_{j})-z,-z,0\}}{1-\alpha}\bigg] \bigg \}.\label{cvar1}
\end{align}
Next, the following proposition can be used to reformulate the distributionally robust risk constraint~\eqref{DRMPCcons4} in a conservative manner, which is suitable for our purpose of limiting the risk of unsafety:

\begin{proposition}\label{prop:risk}
Suppose that $\mathbb{W}=\mathbb{R}^{n}$. Then, the following  inequality holds:
\begin{equation}\nonumber
\begin{split}
 \sup_{\mathrm{P}_t\in\mathbb{D}_{t}}\; \cvar_{\alpha}^{\mathrm{P}_t}&[\dist(y(t),\mathcal{Y}(t)]\\
  \leq  \inf_{z,\lambda, s, {\rho}} \; & z+\frac{1}{1-\alpha}\Bigg[\lambda \theta+\sum_{i=1}^{N} s_i\Bigg] \\
    \mbox{\text{s.t.}} \; &\langle \rho_i, \tilde{c}_t^{(i)} y(t)+ \tilde{d}_t^{(i)}\rangle\leq s_i+z\\
    &s_i+z \geq 0\\
    &s_i \geq 0\\
&\sum_{j=1}^m\rho_{i,j}^2 \sum_{l=1}^{n_y}(y_{l}^2+1) \leq \lambda^2\\
&\lambda \geq 0 \\
    &\langle\rho_i,e\rangle=1\\
    &\rho_i\geq 0\\
    &z\in \mathbb{R},
\end{split}
\end{equation}
where all the constraints hold for $i=1,\dots,N$, and $e\in \mathbb{R}^{m}$ is a vector of all ones. $\rho_{i,j}$ represents the $j$th element of $\rho_i$ and $y_l$ is the $l$th element of $y$. 
\end{proposition}
Its proof follows directly from Lemma 2, \cite[Proposition 1]{hakobyan2020wasserstein}, and~\cite[Theorem 4.2]{esfahani2018data}.
The assumption that $\mathbb{W} = \mathbb{R}^n$ can be relaxed using~\cite[Proposition 1]{hakobyan2020wasserstein}.
Note that the optimization problem on the right-hand side is finite-dimensional, unlike the original one on the left-hand side. 
Thus, by limiting this upper-bound of the distributionally robust safety risk instead of \eqref{DRMPCcons4}, we can completely remove the infinite-dimensionality issue inherent in the original DR-MPC problem~\eqref{DRMPC}. 

Specifically, according to Proposition~\ref{prop:risk}, the DR-MPC problem~\eqref{DRMPC} can be reformulated  as follows:
\begin{subequations}
\begin{align}
\inf_{\substack{\bold{u}, \bm{\xi},\bold{y},\bold{z},\\{\lambda}, {s}, {\rho}}} \;  & J(\xi(t), \bold{u}) 
:= \sum_{k=0}^{K-1} r (\xi_k, u_k)+q(\xi_K) \\
\mbox{s.t.} \; & \xi_{k+1} = f(\xi_k,u_k) \label{GPDRMPC_const1}\\
& y_k=h(\xi_k,u_k)\label{GPDRMPC_const2}\\
& \xi_0 = \xi(t)\label{GPDRMPC_const3} \\
&z_{k}+\frac{1}{1-\alpha}\Bigg[\lambda_{k} \theta+\frac{1}{N_k}\sum_{i=1}^{N_k} s_{k,i}\Bigg]\leq \delta\label{GPDRMPC_const4} \\
&\langle\rho_{k,i}, \tilde{c}^{(i)}_{k}y_k +\tilde{d}_{k}^{(i)}\rangle \leq s_{k,i}+z_{k} \label{GPDRMPC_const5}\\
& s_{k,i}+z_{k} \geq 0 \label{GPDRMPC_const6}\\
& s_{k,i} \geq 0\label{GPDRMPC_const7} \\
&\sum_{j=1}^m\rho_{k,i,j}^2 \sum_{l=1}^{n_y}(y_{k,l}^2+1) \leq \lambda_k^2\label{GPDRMPC_const8}\\
&\lambda_k \geq 0\label{GPDRMPC_const9} \\
&\langle\rho_{k,i},e\rangle=1 \label{GPDRMPC_const10}\\
&\rho_{k,i}\geq 0\label{GPDRMPC_const11} \\
&z_{k}\in\mathbb{R},\label{GPDRMPC_const12}\\
&\xi_k\in\Xi,\; u_k\in\mathcal{U},\label{GPDRMPC_const13}
\end{align}\label{GPDR_MPC}
\end{subequations}
where \eqref{GPDRMPC_const1} and $u_k \in \mathcal{U}$ in \eqref{GPDRMPC_const13} should hold for $k=0,\dots,K-1$, \eqref{GPDRMPC_const2} should hold for $k=0,\dots,K$, and all the other constraints should be satisfied for $k=1,\dots,K$ and $i=1,\dots,N$.
As desired, the reformulated problem is finite-dimensional unlike the original one~\eqref{DRMPC}.
However, it is a nonconvex optimization problem due to the constraints \eqref{GPDRMPC_const5} and \eqref{GPDRMPC_const8} even when the system dynamics and the output equation are affine and the cost function is convex. 
A locally optimal solution to this problem can be efficiently computed by using existing nonlinear programming algorithms such as interior-point methods (e.g., \cite{Nocedal2006}).

\begin{algorithm}[t]
\SetKw{Input}{Input:}
\Input $\xi(t), \mathrm{x}_o(t),  \hat{\mathrm{x}}^{(i)}, \hat{\mathrm{v}}^{(i)}$, $i=1, \ldots, M$;\\
$\mathcal{D}_j :=\big\{(\hat{\mathrm{x}}^{(i)},\hat{\mathrm{v}}_j^{(i)} ),\;i=1,\dots,M\big\},\;j=1,\dots,n_{\mathrm{x}}$\;
Initialize $\mu_{\mathrm{x}}(0) :=\mathrm{x}_o(t)$, $\Sigma_\mathrm{x}(0):=\mathbf{0}$\;
\For{$k=0:K-1$}{
Compute $\tilde{\mu}_\mathrm{v}(k)$, $\tilde{\Sigma}_\mathrm{v}(k)$ and $\tilde{\Sigma}_{\mathrm{xv}}$ from \eqref{Taylor_app}\;
Update $\mu_\mathrm{x}(k+1)$ and $\Sigma_\mathrm{x}(k+1)$ from \eqref{update_eq}\;
Generate a sample $\{\tilde{\mathrm{x}}^{(1)}_o(k+1),\dots,\tilde{\mathrm{x}}^{(N)}_o(k+1)\}$ from $\mathcal{N}(\mu_\mathrm{x}(k+1),\Sigma_\mathrm{x}(k+1))$\;
Compute $\tilde{c}_{k+1}^{(i)}$ and $\tilde{d}_{k+1}^{(i)}$,  $i=1,\dots,N$ using \eqref{gG_eq} and \eqref{cd}\;
}
Solve \eqref{GPDR_MPC} to obtain $\bold{u^*}$\;
\Return $u(t)=u_0^*$\;
\caption{Learning-based DR-MPC at stage $t$}\label{GPDR_MPC_alg}
\end{algorithm}

The overall  learning-based DR-MPC at stage $t$ is shown in Algorithm \ref{GPDR_MPC_alg}. 
At each stage, the current states of the robot and the obstacle as well as $M$ past observations $\{(\hat{\mathrm{x}}^{(i)}, \hat{\mathrm{v}}^{(i)}) \}_{i=1}^M$ of the obstacle's position and velocity
are taken as the input data. 
Then, the obstacle's movement for future stages is learned by GP regression, and is used in the DR-MPC problem \eqref{GPDR_MPC}. The first element of locally optimal solution $\bold{u}^*$ is taken as  the motion control action for the robot at the current stage. 
Note that at stage $t = 0$, the dataset $\mathcal{D}$ consists of all zeros. As time goes on, new observations are added to  the  dataset for GP regression.  During the update, old observations are removed so that only $M$ latest data are stored.

\section{Experiment Results}\label{s5}

In this section, we present simulation results to demonstrate the performance of our motion control method. 
In our experiments, we consider a car-like vehicle navigating a 2D environment with the following bicycle dynamics~\cite{polack2017kinematic}:
\begin{equation}
\begin{split}
&x^v(t+1)=x^v(t)+T_s v^v(t) \cos(\theta^v(t)+\beta^v(t))\\
&y^v(t+1)=y^v(t)+T_s v^v(t)\sin(\theta^v(t)+\beta^v(t))\\
&\theta^v(t+1)=\theta^v(t)+T_s v^v(t)\frac{\sin(\beta^v(t))}{l_{r}}\\
&\beta^v(t+1)=\beta^v(t)+T_s \tan^{-1}\Big(\frac{l_{r}}{l_{r}+l_{f}}\tan\delta(t)\Big),
\end{split}\label{veh_dyn}
\end{equation}
where $x^v(t)$ and $y^v(t)$ are the coordinates of the vehicle's center of gravity, $\theta^v(t)$ is the heading angle, $\beta^v(t)$ is the current velocity angle. The control inputs are  velocity $v^v(t)$ and steering angle $\delta^v(t)$. The coefficients $l_{f}$ and $l_{r}$ represent the distances from the center of gravity to the front and rear wheels, respectively. Throughout the simulations, we assume that $l_{f}=l_{r}=2$.
We also impose the following state and control constraints:
\begin{equation*}
v_v(k)\in [0,30],\quad u_v(k)\in [-\pi/6,\pi/6] \quad \forall k.
\end{equation*}

The vehicle is controlled to follow the centerline of the track, while avoiding two dynamic obstacles. 
The centerline is thus taken as the reference trajectory $y^{ref}$ in \eqref{cost}. The two obstacles are  rectangular car-like vehicles with size $2\times 1$. It is straightforward to check that for both obstacles $g_k$ and $G_k$ are easily found from the state that consists of the vehicle's center of mass and its heading angle. 
 In our experiments, we set $Q=P=I$ and $R=0.01 I$. 
 The sampling time $T_s$ and $T_o$ are set to be $0.01$, and the MPC horizon is chosen as $K=5$. The risk tolerance level and the confidence level were selected as $\delta=0.01$ and $\alpha=0.95$, respectively.

To evaluate the performance of learning-based DR-MPC, we compare it to its non-robust counterpart obtained by sample average approximation (SAA)~\cite{hakobyan2019risk}.
All the simulations were conducted on a PC with 3.70 GHz Intel Core i7-8700K processor and 32 GB
RAM. The optimization problem was modeled in AMPL~\cite{fourer1990modeling} and solved using interior-point method-based solver IPOPT~\cite{wachter2006implementation}.

\begin{figure*}[t]
\begin{subfigure}{0.48\linewidth}
  \centering
  \includegraphics[width=\linewidth]{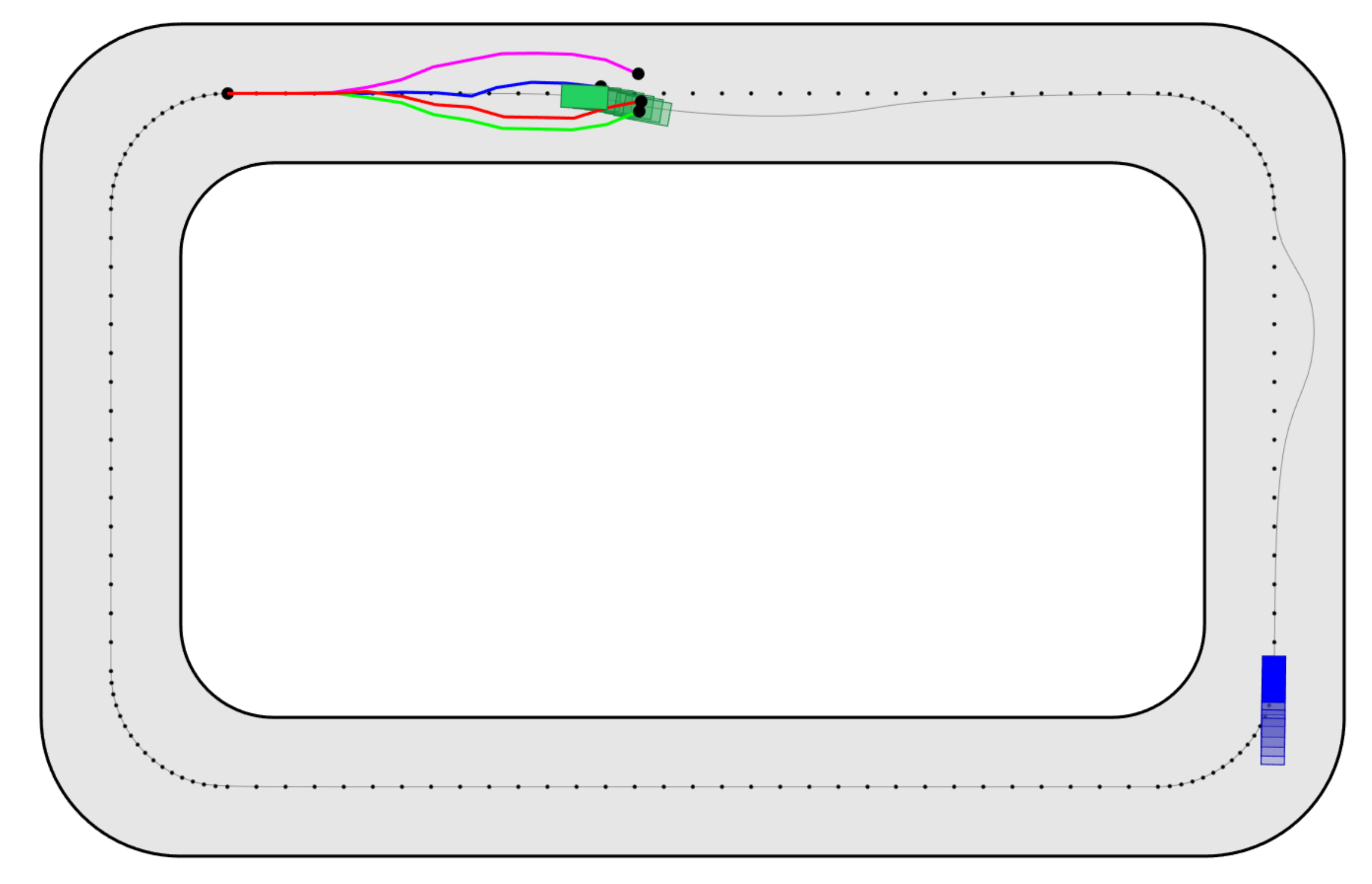}
  \subcaption{$t=13$}
  \label{exp11}
\end{subfigure}%
\begin{subfigure}{0.48\linewidth}
  \centering
  \includegraphics[width=\linewidth]{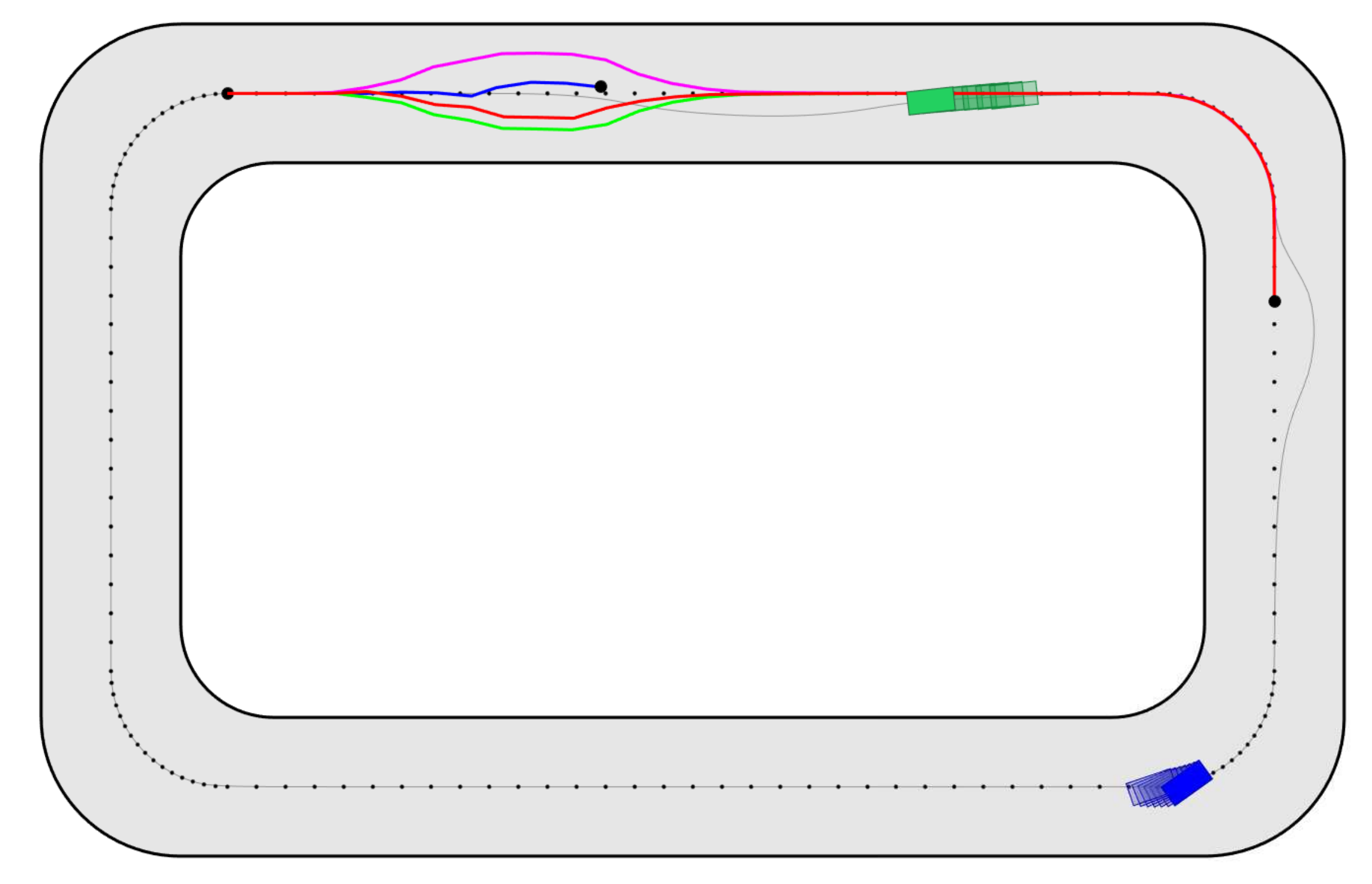}
    \subcaption{$t=38$}
    \label{exp12}
\end{subfigure}\\
\begin{subfigure}{0.48\linewidth}
  \centering
  \includegraphics[width=\linewidth]{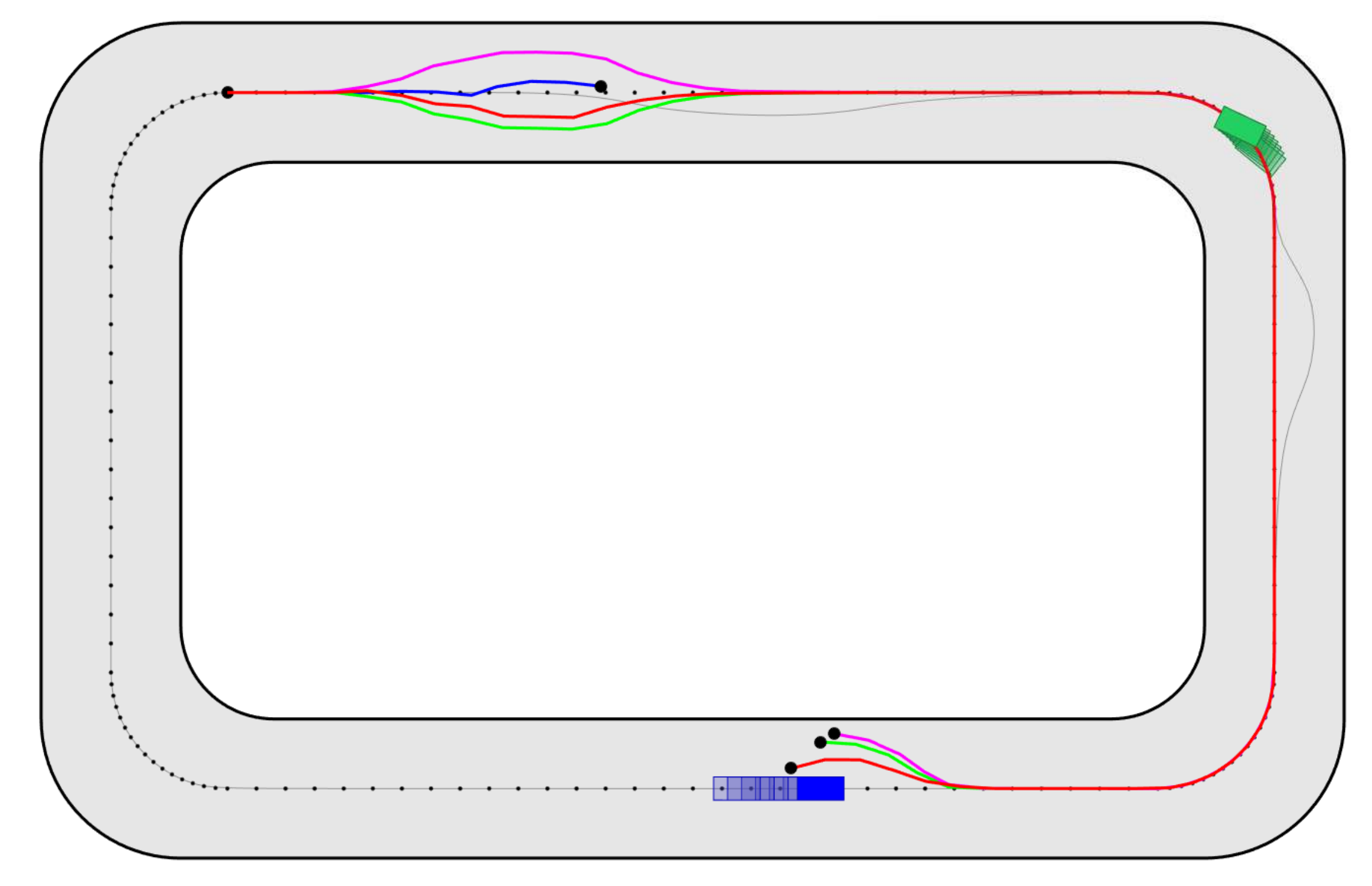}
    \subcaption{$t=67$}
    \label{exp13}
\end{subfigure}%
\begin{subfigure}{0.48\linewidth}
  \centering
  \includegraphics[width=\linewidth]{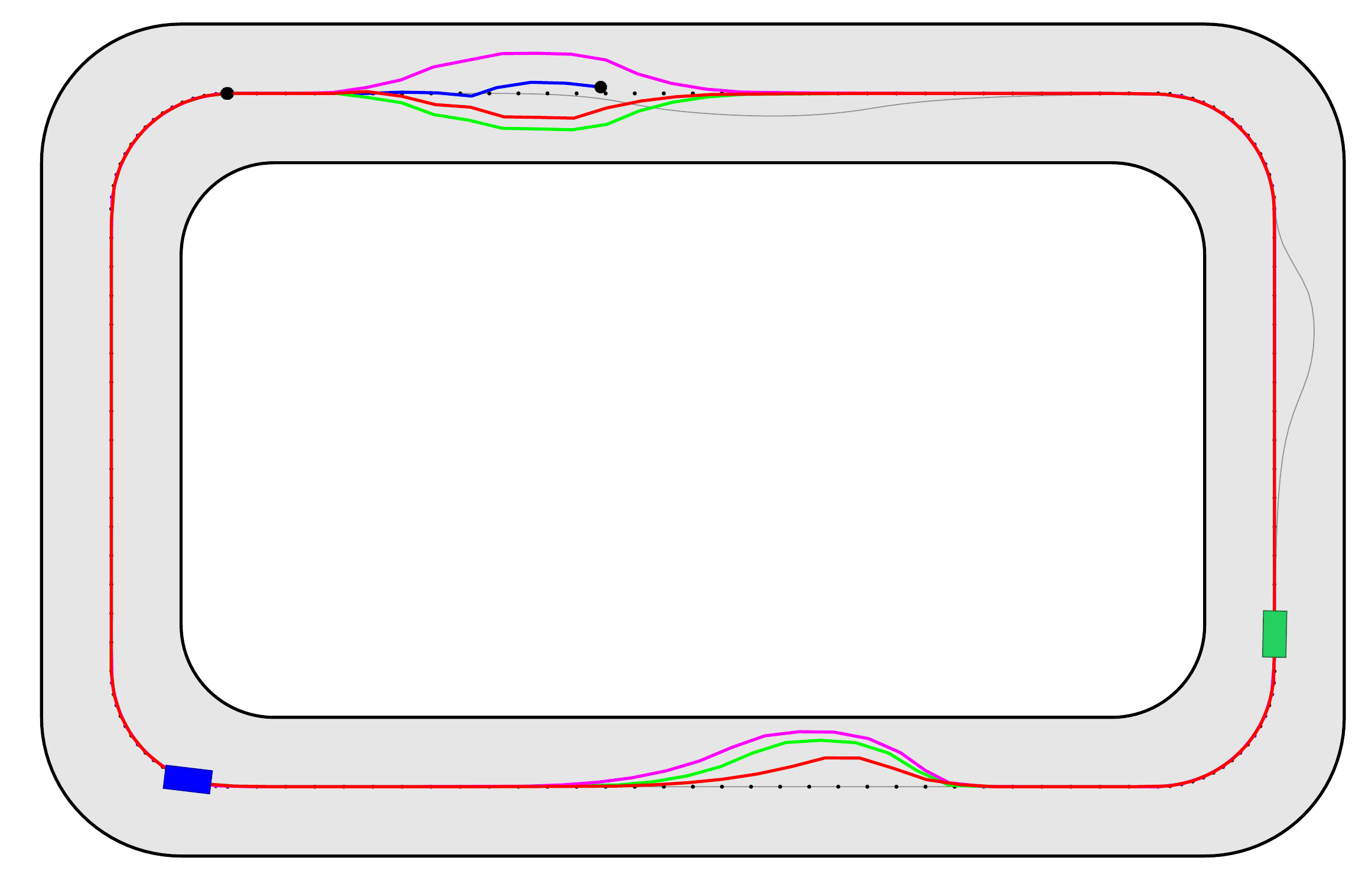}
    \subcaption{$t=114$}
    \label{exp14}
\end{subfigure}\\
\begin{subfigure}{\linewidth}
  \centering
  \vspace{0.1in}
  \includegraphics[width=0.6\linewidth]{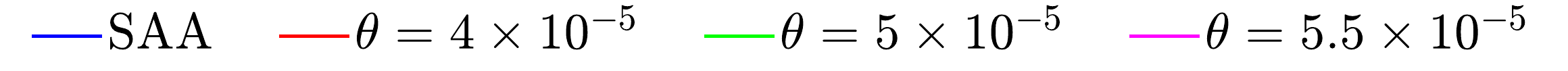}    
\end{subfigure}
\caption{The  trajectories of the vehicle controlled by SAA-MPC and DR-MPC
with $\theta=4\times 10^{-5},5\times 10^{-5}$, and $5.5\times 10^{-5}$. 
The current vehicle position is marked with a black dot. 
The green and blue rectangles represent the two obstacles, while the the transparent ones are the $K$ steps-ahead prediction of the obstacles, obtained via GP regression. 
The reference centerline for the vehicle is displayed with points, while
the thin grey curve is the actual trajectory of the obstacles.}\label{exp1}
\end{figure*}

Fig.~\ref{exp1} shows the resulting trajectories for different sizes of the Wasserstein ambiguity set  compared to the SAA version (SAA-MPC) with $N=50$ samples. At each stage, the dataset for GP regression is updated to keep only the latest $M=20$ observations.

In early stages, the robotic vehicle follows the centerline while predicting the future motion of the obstacles. As shown in Fig.~\ref{exp11}, when reaching one obstacle that abruptly changes its heading angle
at $t=13$, the vehicle tries to avoid it. 
In the case of SAA-MPC, the vehicle collides with the obstacle
because the distributional information learned by GP regression is inaccurate. 
As a result, the risk constraint is violated and the MPC problem becomes infeasible. 
Meanwhile, the vehicle controlled by our method successfully bypasses the obstacle. 
The safety margin increases with the radius $\theta$ of the Wasserstein ambiguity set.

Fig.~\ref{exp12} shows the situation at $t=38$, where the vehicle controlled by our method continues to follow the reference trajectory for all $\theta$'s. 
Meanwhile, the GP is not well enough to be able to predict the motion of the obstacles around the corners,  although it shows a good performance when there is no sudden change in the obstacle's movement.
As shown in Fig.~\ref{exp13},
at $t=67$ the second obstacle  interferes with the path of the vehicle. Similar to the previous obstacle, the vehicle controlled by DR-MPC avoids the obstacle for all $\theta$'s. 
In the case of the smallest radius $\theta=4\times 10^{-5}$, the vehicle chooses to take aggressive action while satisfying the risk constraint. 
As the Wasserstein ambiguity set increases, i.e., $\theta$ increases, 
the robot makes a more conservative (i.e., safer) decision, inducing a bigger safety margin. 
Fig.~\ref{exp14} displays the trajectories for all cases after the vehicle completes one lap.  Note that only the non-robust SAA version failed to complete the lap due to collision, while our method succeeded to do so for all $\theta$'s.  

In summary,  we conclude that the proposed distributionally robust method successfully preserves safety even with moderate errors in the learning results.
In the case of very small ambiguity sets (e.g., $\theta = 4\times 10^{-5}$), the resulting control action may be too aggressive to guarantee safety when the learning errors are significant.
Whereas, for $\theta=5.5\times 10^{-5}$, the vehicle deviates too much from the reference trajectory, inducing a large cost. 
Based on our experiments, $\theta=5\times 10^{-5}$ may be selected for a good tradeoff between safety and cost.

Table~\ref{Table1} shows the accumulated cost and the amount of time for completing one lap on the track, and the average computation time required for solving a single DR-MPC problem~\eqref{GPDR_MPC}. As expected, both of the total cost and the lap time increase with $\theta$ since the vehicle controlled by DR-MPC with larger $\theta$ is more conservative and deviates further from the reference trajectory. 
Computation time is small in all cases although a nonconvex optimization problem is solved in each iteration. This result shows the potential of using our distributionally robust method in real-time applications.  

\begin{table}[!t]
\caption{Accumulated cost, lap time, and average computation time for the nonlinear car-like vehicle motion control with $N=50$, $\delta=0.01$, and $\alpha=0.95$.}
\centering
\setlength{\tabcolsep}{0.5em} 
\begin{tabular}{>{\raggedright\arraybackslash} m{2.6cm}| >{\centering\arraybackslash} m{0.8cm} | c c c}
\hline
& \multirow{2}{*}{SAA} & \multicolumn{3}{c}{DR-MPC ($\theta$)}\\
\cline{3-5}
& & $4\times 10^{-5}$ & $5\times 10^{-5}$ & $5.5\times 10^{-5}$ \\
\hline\hline
\bf{Accumulated Cost} & $+\infty$& 491.79 & 594.68 & 703.59 \\
\hline
\bf{Lap Time (sec)} & -& 105.26 & 109.45 & 110.31\\\hline
\bf{Avg. Run Time (sec)}  & -& 0.6572 & 0.6767 & 0.6942 \\\hline
\end{tabular}
\label{Table1}
\end{table}

\section{Conclusion}
We have proposed a distributionally robust decision-making tool for safe motion control of robotic vehicles in an environment with dynamic obstacles. 
Our DR-MPC method limits the risk of unsafety even with moderate errors in the obstacle's motion predicted by
GP regression. 
For computational tractability, we have also developed a reformulation approach exploiting modern distributionally robust optimization techniques.
The experimental results demonstrate the safety-preserving capability of our method under  moderate learning errors
and the potential for real-time application. 
In the future, the proposed method can be extended to enhance the capability of fast adaptive reactions, especially when considering sudden motion changes, and to address partial observability.

\bibliographystyle{IEEEtran}\bibliography{reference} 

\end{document}